\documentclass[notitlepage, 11pt]{article}
\usepackage[left=1in, right=1in, top=1in, bottom=1in]{geometry}

\usepackage{titling}
\usepackage{natbib}
\setcitestyle{open={(},close={)}} 

\usepackage{authblk}
\pretitle{\centering\Large}
\posttitle{\par}
\predate{}
\postdate{}

\usepackage[utf8]{inputenc} %
\usepackage[T1]{fontenc}    %
\usepackage{hyperref}       %
\hypersetup{
  colorlinks,
  linkcolor={red!50!black},
  citecolor={blue!50!black},
  urlcolor={blue!50!black}
}
\usepackage{url}            %
\usepackage{booktabs}       %
\usepackage{amsfonts}       %
\usepackage{nicefrac}       %
\usepackage{microtype}      %
\usepackage{xcolor}         %
\usepackage{amsmath}
\usepackage{amsthm}
\usepackage{amssymb}
\usepackage{bm}
\usepackage{graphicx}
\usepackage{multirow}
\usepackage{multicol}
\usepackage{adjustbox}
\usepackage{cleveref}
\usepackage{wrapfig}
\usepackage{tikz}

\usetikzlibrary{arrows.meta, positioning, calc, backgrounds, patterns, shapes.geometric, decorations.pathmorphing}

\definecolor{deepblue}{RGB}{35, 75, 105}
\definecolor{softteal}{RGB}{224, 242, 241}
\definecolor{mainteal}{RGB}{0, 121, 107}
\definecolor{vibrantorange}{RGB}{230, 81, 0}
\definecolor{metapurple}{RGB}{106, 27, 154}
\definecolor{subtlegray}{RGB}{250, 250, 250}

\newcommand{\bx}{\mathbf{x}}

\newcommand{\bz}{\mathbf{z}}

\newcommand{\bI}{\mathbf{I}}
\newcommand{\beps}{\bm{\epsilon}}

\newcommand{\E}{\mathbb{E}}
\newcommand{\cL}{\mathcal{L}}
\newcommand{\z}[1]{\bz_{t({#1})}}
\newcommand{\zs}[1]{\bz_{s({#1})}}
\newcommand{\kl}{\mathrm{KL}}
\newcommand{\KL}[2]{\mathrm{KL}({#1}\|{#2})}
\newcommand{\snr}{\text{SNR}}
\newcommand{\twonorm}[1]{\|{#1}\|_2}
\newcommand{\diff}{\mathop{}\!\mathrm{d}}

\ifdefined\nonewproofenvironments\else
\ifdefined\ispres\else
\newtheorem{theorem}{Theorem}

\theoremstyle{definition}

\renewenvironment{proof}{\noindent\textbf{Proof}\hspace*{1em}}{\qed\\}

\title{\LARGE Demystifying Diffusion Objectives: Reweighted Losses are \\Better Variational Bounds}

\author[]{Jiaxin Shi}
\author[]{Michalis K. Titsias}
\affil[]{Google DeepMind}
\affil[ ]{\texttt{\{jiaxins,mtitsias\}@google.com}}
\date{}

\begin{document}

\maketitle

\begin{abstract}
We derive a new theoretical interpretation of the 
reweighted losses that are widely used for training diffusion models. 
Our method is based on constructing a cascade of time-dependent variational lower bounds on the data log-likelihood, that provably improves upon the standard evidence lower bound and results in reduced data-model KL-divergences.
Combining such bounds gives rise to reweighted objectives that can be applied to any generative diffusion model including both continuous Gaussian diffusion and masked (discrete) diffusion models. 
Then, we showcase this framework in masked diffusion and report significant improvements over previous training losses in pixel-space image modeling, approaching sample quality comparable to continuous diffusion models.
Our results also provide a theoretical justification for the simple weighting scheme widely used in masked image models. 
\end{abstract}

\section{Introduction}

Diffusion models \citep{sohl2015deep,song2019generative,ho2020denoising,song2020score} have become the prominent generative modelling approach for  image generation~\citep{rombach2022high,ramesh2022hierarchical,saharia2022photorealistic}, audio~\citep{chen2021wavegrad,kong2021diffwave} and video generation~\citep{ho2022video,villegas2023phenaki}.
These methods have been also extended to discrete data  \citep{sohl2015deep,austin2021structured,hoogeboom2021argmax,campbell2022continuous,lou2023discrete} and applied to domains such as language modeling~\citep{nie2025large,ye2025dream}, graph generation \citep{vignac2023digress}, text-to-sound generation \citep{yang2023diffsound} or protein design \citep{wang2025finetuning}. 
Very recent work on masked diffusion \citep{shi2024simplified} suggests that the use of discrete diffusion enables a multimodal generative framework that unifies the treatment of diverse data, including images and text.

The standard framework to train a diffusion model, either continuous or discrete, is to view it as as a probabilistic latent-variable model and apply approximate Maximum-Likelihood  learning by optimizing the Evidence Lower Bounds (ELBOs) on the data log-likelihood. A detailed presentation of the ELBO and its continuous-time limit can be found in \citet{kingma2021variational}, while corresponding continuous-time ELBOs  for discrete masked diffusions  
were derived more
recently \citep{shi2024simplified,sahoo2024simple,ou2024}.  
 However, modern experimental practices in continuous Gaussian diffusion have converged to optimizing not the ELBO itself but a reweighted version of the ELBO. 
 The reason for this switch is  noticeably higher perceptual quality reported widely in literature on image generation~\citep{ho2020denoising,nichol2021improved}.

Although the reweighted loss is widely adopted in practice, its theoretical understanding has been quite limited. 
 \citet{kingma2023understanding} proposed to interpret the reweighted loss as a weighted sum of ELBOs of augmented data - where the data augmentation corresponds to Gaussian noises employed in continuous diffusion models. 
However, this derivation does not explain why ELBOs on the noise-perturbed data provide better signals for learning the denoiser.
It is also unclear how to extend such result to discrete diffusion models.

In this work, we give a new interpretation of the reweighted loss.
Our initial motivation is to reduce   the accumulation of errors when  approximating (during training) the backward denoising process. To this end,  we prove that the standard ELBO on the overall log-likelihood is not the best Maximum-Likelihood objective to train the model up to a given denoising time, but instead there is another improved time-dependent ELBO having smaller Kullback-Leibler divergence. Based on this observation, we suggest to train a full diffusion model not with the standard ELBO, but instead a weighted sum of all time-dependent ELBOs, visualized in \Cref{fig:objective}.
This analysis provides a theoretical interpretation of the reweighted loss that is not limited to the scope of continuous-time diffusion or Gaussian noise processes. Following this, we propose to adapt the reweighted schemes from continuous diffusion models to train masked  diffusion models. 
As we demonstrate in our results, the new training objectives 
for masked diffusion yield significant improvements in image sample quality, measured by Fréchet inception distance (FID), over previous standard ELBO based losses. 
    
\begin{figure}[t]
\centering
\resizebox{\columnwidth}{!}{%
\begin{tikzpicture}[
	node distance=1.5cm and 1.5cm,
	arrow/.style={->, >=stealth, thick, color=black!80},
	label text/.style={font=\small\sffamily},
	math text/.style={font=\large},
	arrow label/.style={font=\normalsize}, 
	annotation arrow/.style={->, >=stealth, blue, thick},
	zone boundary/.style={draw=blue!20, thick, rounded corners=15pt},
	blue fill style/.style={fill=blue!4},
	white fill style/.style={fill=white},
	]
	
	\def\rowOne{0}
	\def\rowTwo{-1.5}
	\def\rowMid{-3}
	\def\rowLast{-4.5}
	
	\def\colX{0}      %
	\def\colZz{2}     %
	\def\colZo{4}     %
	\def\colZmid{6}   %
	\def\colZpen{8}   %
	\def\colZT{10}    %
	
	\def\arrowTop{0.5}
	\def\arrowBottom{-5.0}
	
	\begin{scope}[on background layer]
		\def\padH{0.8} \def\padV{0.8}
		\coordinate (TL) at (\colX - \padH, \rowOne + \padV); 
		\coordinate (BR) at (\colZT + \padH, \rowLast - \padV);
		\coordinate (SplitTop) at (\colZz - 1.0, \rowOne + \padV); 
		\coordinate (SplitBottom) at (9.0, \rowLast - \padV);
		
		\begin{scope}
			\clip[rounded corners=15pt] (TL) rectangle (BR);
			\fill[blue fill style] (TL) rectangle (BR);
			\fill[white fill style] 
			($(TL) + (-2, 2)$) --     
			($(SplitTop) + (0, 2)$) -- 
			(SplitTop) --              
			(SplitBottom) --           
			($(SplitBottom) + (0, -2)$) -- 
			($(TL |- BR) + (-2, -2)$) -- 
			cycle;
			\draw[zone boundary, sharp corners] (SplitTop) -- (SplitBottom);
		\end{scope}
		\draw[zone boundary] (TL) rectangle (BR);
	\end{scope}

	\node[math text] (r1_x) at (\colX, \rowOne) {$\mathbf{x}$};
	\node[math text] (r1_z0) at (\colZz, \rowOne) {$\mathbf{z}_{t(0)}$};
	\node[math text] (r1_z1) at (\colZo, \rowOne) {$\mathbf{z}_{t(1)}$};
	\node[math text] (r1_dots) at (\colZmid, \rowOne) {$\dots$};
	\node[math text] (r1_zprev) at (\colZpen, \rowOne) {$\mathbf{z}_{t(T-1)}$};
	\node[math text] (r1_zT) at (\colZT, \rowOne) {$\mathbf{z}_{t(T)}$};
	
	\draw[arrow] (r1_z0) -- node[above, arrow label] {$q$} (r1_x);
	\draw[arrow] (r1_z1) -- node[above, arrow label] {$p_\theta$} (r1_z0);
	\draw[arrow] (r1_dots) -- node[above, arrow label] {$p_\theta$} (r1_z1);
	\draw[arrow] (r1_zprev) -- node[above, arrow label] {$p_\theta$} (r1_dots);
	\draw[arrow] (r1_zT) -- node[above, arrow label] {$p_\theta$} (r1_zprev);
	
	\node[math text] (r2_x) at (\colX, \rowTwo) {$\mathbf{x}$};
	\node[math text] (r2_z1) at (\colZo, \rowTwo) {$\mathbf{z}_{t(1)}$}; 
	\node[math text] (r2_dots) at (\colZmid, \rowTwo) {$\dots$}; 
	\node[math text] (r2_zprev) at (\colZpen, \rowTwo) {$\mathbf{z}_{t(T-1)}$};
	\node[math text] (r2_zcurr) at (\colZT, \rowTwo) {$\mathbf{z}_{t(T)}$};
	
	\draw[arrow] (r2_z1) -- node[above, arrow label] {$q$} (r2_x);
	\draw[arrow] (r2_dots) -- node[above, arrow label] {$p_\theta$} (r2_z1);
	\draw[arrow] (r2_zprev) -- node[above, arrow label] {$p_\theta$} (r2_dots);
	\draw[arrow] (r2_zcurr) -- node[above, arrow label] {$p_\theta$} (r2_zprev);
	
	\node at ($(r2_x)!0.5!(r2_x |- 0, \rowLast)$) {$\vdots$};
	\node at ($(r2_zcurr)!0.5!(r2_zcurr |- 0, \rowLast)$) {$\vdots$};
	
	\node[math text] (rT_x) at (\colX, \rowLast) {$\mathbf{x}$};
	\node[math text] (rT_zprev) at (\colZpen, \rowLast) {$\mathbf{z}_{t(T-1)}$};
	\node[math text] (rT_zT) at (\colZT, \rowLast) {$\mathbf{z}_{t(T)}$};
	
	\draw[arrow] (rT_zprev) -- node[above, arrow label] {$q$} (rT_x);
	\draw[arrow] (rT_zT) -- node[above, arrow label] {$p_\theta$} (rT_zprev);
	
	\def\graphX{-5.0} 
	
	\draw[thick, black!60] (\graphX, \arrowTop) -- (\graphX, \arrowBottom); 
	
	\node[left, font=\large] at (\graphX, \rowOne) {$w_1$};
	\draw[fill=white, pattern=north east lines, pattern color=black!40, draw=black!60] (\graphX, \rowOne - 0.15) rectangle ++(0.8, 0.3);
	
	\node[left, font=\large] at (\graphX, \rowTwo) {$w_2$};
	\draw[fill=white, pattern=north east lines, pattern color=black!40, draw=black!60] (\graphX, \rowTwo - 0.15) rectangle ++(1.5, 0.3);
	
	\node[left, font=\large] at (\graphX, \rowLast) {$w_T$};
	\draw[fill=white, pattern=north east lines, pattern color=black!40, draw=black!60] (\graphX, \rowLast - 0.15) rectangle ++(1.0, 0.3);
	
	\draw[blue, thick] plot [smooth, tension=0.8] coordinates {
		(\graphX, 0.2)          
		(\graphX + 0.5, -0.5)   
		(\graphX + 1.8, -1.8)   
		(\graphX + 0.8, -3.8)   
		(\graphX + 0.3, -5.0)   
	};
	
	\draw[annotation arrow] 
	(\graphX + 2.5, \arrowTop) -- 
	node[midway, sloped, below, color=black] {Improved ELBO} 
	(\graphX + 2.5, \arrowBottom);
	
	\node[font=\large, anchor=west] at (\graphX + 2.6, \rowOne) {$\mathcal{L}^{(1)}(\mathbf{x})$};
	\node[font=\large, anchor=west] at (\graphX + 2.6, \rowTwo) {$\mathcal{L}^{(2)}(\mathbf{x})$};
	\node[font=\large, anchor=west] at (\graphX + 2.6, \rowLast) {$\mathcal{L}^{(T)}(\mathbf{x})$};
	
	\draw[annotation arrow] 
	(\colZT + 1, \arrowBottom) -- 
	node[midway, sloped, below, color=black] {Tractable Sampling} 
	(\colZT + 1, \arrowTop);
	
	\begin{scope}[shift={(5, -5.5)}, scale=0.4]
		\draw[<->, black!60] (0, 1.5) -- (0,0) -- (2,0);
		\draw[blue, thick] (0,0) to[out=10, in=200] (1.5, 1.2);
	\end{scope}
	
	\draw[dashed, blue, ->, thick] 
	(\graphX + 0.4, -5.0) 
	to[out=-10, in=175]     
	node[midway, above, color=black, font=\footnotesize] {C.D.F} 
	(4.9, -5.3); 
	
	\node[below=1.1cm, align=center, text width=17cm] at (3, \rowLast) {
		\large Diffusion objectives: $ \mathcal{L}^{\tilde{w}}(\mathbf{x}) = {\displaystyle \lim_{T\to \infty}} \sum_{i=1}^T w_i \mathcal{L}^{(i)}(\mathbf{x})  =  \int_0^1 \textcolor{blue}{\tilde{w}(t)} \mathbb{E}_{q(\mathbf{z}_t|\mathbf{x})} \left[ L_{\text{denoise}}(\mathbf{z}_t, \mathbf{x}, t) \right] dt + \mathrm{C}$
	};
	
\end{tikzpicture}%
}
\caption{Diffusion objectives viewed as a weighted sum of the ELBOs of a sequence of models with optimal decoders (defined in \Cref{sec:optimal-decoders}). For continuous Gaussian diffusion models: $L_{\text{denoise}}(\mathbf{z}_t, \mathbf{x}, t) = \frac{1}{2} \lambda'(t) \twonorm{\beps - \beps_\theta(\bz_t, t)}^2 $. For masked diffusion models: $L_{\text{denoise}}(\mathbf{z}_t, \mathbf{x}, t) = -\frac{\alpha_t'}{1 - \alpha_t} \delta_{\bz_t, m}\cdot \bx^\top \log \mu_\theta(\bz_t)$. }
\label{fig:objective}
\end{figure}

\subsection{Related work}

The standard justification for using a reweighted loss over a likelihood-based loss (such as the ELBO) is that it prioritizes perceptually relevant signals over high frequency details~\citep{dieleman2024noise}. 
A similar argument is often used to explain the suboptimal sample quality of autoregressive image models despite their strong likelihood performance.
We offer an alternative perspective: the efficacy of reweighted loss is due to a more fundamental improvement in the ELBOs, specifically achieved by yielding smaller KL divergences.

Parallel with the development of discrete diffusion models, masked image models such as MaskGIT also proposed using a weighted sum of cross-entropy denoising losses at many noise levels~\citep{chang2022maskgit}. 
\citet{li2024autoregressive} further extended this method to modeling image latent-space~\citep[e.g., Stable Diffusion latents,][]{rombach2022high} by switching from cross-entropy to euclidean losses resulted from local continuous diffusion.
Despite the similarity between such models and masked diffusion models~\citep{zheng2024masked}, they often use a simple heuristic weighting scheme---denoising losses on mask inputs are summed over minibatches, divided by the total number of masks in the batch. 
\citet{you2025effective} summarizes the connection and differences of such models and masked diffusion models and proposed a hybrid model that uses simple weighting.
In \Cref{sec:discrete-diffusion} we show that the simple weighting alone can be explained as a special case of our framework, and, when it is applied to masked diffusion models in isolation, it leads to a significant improvement in sample quality without requiring any other modifications as in masked image models~\citep{chang2022maskgit,li2024autoregressive,li2025fractal} and \citet{you2025effective}.

\section{Background: Diffusion Models}

We %
consider the task of generative modeling: 
Given a dataset of observations $\bx$ with an underlying distribution $q(\bx)$, we aim to train a probabilistic model $p_\theta(\bx)$ that approximates $q(\bx)$. 
After training, we can draw novel samples from $p_\theta(\bx)$ that resemble observations in the dataset.

The class of models we will look into is diffusion models~\citep{sohl2015deep,song2019generative,ho2020denoising,song2020score}. 
Typically, 
we construct such models by first introducing a ``forward'' noising process that gradually transforms $\bx$ to %
noise.
By reversing this process, we obtain a generative model that creates data from noise.

Following \citet{kingma2021variational}, we define the forward process as a sequence of random variables $\bz_t$ indexed by time $t$ between $[0, 1]$, where $\bz_t$ represents the noise-perturbed data at time $t$.
For continuous observation $\bx$, a Gaussian noise process is commonly employed. 
In this case, the marginal distribution of $\bz_t$ is given by
\begin{align}
    q(\bz_t|\bx) =  \mathcal{N}(\bz_t|\alpha_t \bx, \sigma_t^2 \bI), 
\end{align}
where $\alpha_1 \approx 0$ and $\sigma_1 \approx 1$ such that $\bz_1$ follows a standard normal distribution. 
It is common to parameterize the forward process with respect to the signal-to-noise ratio (SNR) or log-SNR:
\begin{align}
    \snr(t) \triangleq \alpha_t^2 / \sigma_t^2,\quad \lambda(t) = \log \snr(t).
\end{align}
The transition from any time $s$ to $t$ ($s < t$) also follows a Gaussian distribution:
\begin{align}
    q(\bz_t|\bz_s) = \mathcal{N}\left(\bz_t\Big |\frac{\alpha_t}{\alpha_s}\bz_s, (1 - \kappa_{s,t})\sigma_t^2 \bI\right) \text{\quad where \quad} \kappa_{s,t} \triangleq \frac{\snr(t)}{\snr(s)}. 
\end{align}
A diffusion model seeks to revert the forward process, yielding a generative process that runs from time $1$ to $0$. 
To approximate this process, we introduce the reverse model $p_\theta(\bz_s|\bz_t)$ of the transition distribution from any time $t$ to $s$.
To derive the training objective of this reverse model, we define the discrete-time generative model by looking at finite time points $t(i)= i / T \in [0, 1]$, where $i=0,\ldots,T$.  
The joint probability distribution of the discretized %
model is
\begin{align} \label{eq:generative-joint}
    p_\theta(\bx, \z{0:T}) = p(\bx|\z{0})\prod_{i=0}^T p_\theta(\zs{i}|\z{i}),
\end{align}
where we let $s(i) = (i - 1)/T$. 
A standard derivation~\citep{sohl2015deep} gives the discrete-time evidence lower bound (ELBO) on the data log-likelihood,
\begin{align} \label{eq:elbo}
    \log p_\theta(\bx) 
    \geq \cL_T(\bx) &= \mathbb{E}_{q(\z{0}|\bx)}[\log p_\theta(\bx|\z{0})] - \KL{q(\z{T}|\bx)}{p(\z{T})} \notag \\
    & - \sum_{i=0}^T \mathbb{E}_{q(\z{i}|\bx)} [\KL{q(\zs{i}|\z{i}, \bx)}{p_\theta(\zs{i}|\z{i})}].
\end{align}
The reverse model is often chosen to mimic the structure of the true reverse distribution: $p_\theta(\bz_s|\bz_t) \triangleq q(\bz_s|\bz_t, \bx = \mu_\theta(\bz_t, t))$, using a neural network $\mu_\theta$ to predict the clean data (thus known as a ``denoiser'').
For Gaussian diffusion, %
we can obtain the following form of $q(\bz_s|\bz_t, \bx)$ through  Bayes' rule:
\begin{align}
    q(\bz_s|\bz_t, \bx) = \mathcal{N}\left(\bx_s\Big |(1 - \kappa_{s,t})\frac{\alpha_s}{\alpha_0} \bx + \kappa_{s,t}\frac{\alpha_s}{\alpha_t} \bx_t, \sigma_s^2 (1 - \kappa_{s,t})\bI\right).
\end{align}
In this case, one can show that the KL divergence terms in the ELBO simplify as
\begin{align} \label{eq:kl-t}
    \KL{q(\bz_s|\bz_t, \bx)}{p_\theta(\bz_s|\bz_t)} = \frac{1}{2}(\snr(s) - \snr(t))\twonorm{\bx - \mu_\theta(\bz_t, t)}^2.
\end{align}
Originated from \citet{ho2020denoising} to mimic the denoising score matching parameterization \citep{song2019generative}, the widely-used $\beps$-parameterization  leverages the noise form of $q(\bz_t|\bx)$: $\bz_t = \alpha_t \bx + \sigma_t \beps, \beps \sim \mathcal{N}(\mathbf{0}, \bI)$ and lets $\mu_\theta(\bz_t, t) \triangleq (\bz_t - \sigma_t \beps_\theta(\bz_t, t))/\alpha_t$. As shown by \cite{kingma2021variational}, in the continuous-time limit ($T\to \infty$) the ELBO in \eqref{eq:elbo} becomes
\begin{align}
    \cL_{\infty}(\bx) = \frac{1}{2}\int_0^1  \lambda'(t)\E_{\beps \sim \mathcal{N}(\mathbf{0}, \bI)}\left[ \twonorm{\beps - \beps_\theta(\bz_t, t)}^2\right] \diff t. 
\end{align}

\paragraph{Weighted losses.}
Although the ELBO seems a reasonable objective for training diffusion models, in practice 
reweighted versions of the ELBO  
empirically lead to better perceptual quality~\citep{ho2020denoising,nichol2021improved}. 
These %
objectives can be expressed as
\begin{align}
    \cL^{\tilde{w}}(\bx) 
    =\; &\frac{1}{2}\int_0^1 \tilde{w}(t) \lambda'(t) \E_{\beps \sim \mathcal{N}(\mathbf{0}, \bI)}\left[\twonorm{\beps - \beps_\theta(\bz_t, t)}^2\right] \diff t. \label{eq:weighted-loss-eps}
\end{align}
\citet{ho2020denoising} set the weight function as $\tilde{w}(t) = \frac{1}{\lambda'(t)}$ and  pointed out that it leads to higher sample quality measured by FID than the ELBO objective. 
This reweighted loss (also known as ``simple'' objective) is currently widely used.
\citet[Table 1]{kingma2023understanding} gives a full characterization of the various weighting functions proposed in the literature.

\section{Diffusion Models with Optimal Decoders}
\label{sec:optimal-decoders}

To understand the reweighted objective, our first observation is that the standard ELBO in \Cref{eq:elbo} uses the denoiser to construct the reverse transition distributions at all timesteps. 
However, there are other choices we can make about the generative model by mixing the denoiser with an ``optimal decoder'' introduced below. 

We define the following reverse generative model, where we replace the approximate reverse transition distributions between $\bx$ and $\z{i}$ with an ``optimal decoder'' $q(\bx|\bz_{t(i)})$, which is the ground truth reverse transition distribution satisfying $q(\bx|\z{i}) = \frac{q(\z{i}|\bx)q(\bx)}{q(\z{i})}$. 
The corresponding joint distribution is
\begin{align}
    p_\theta(\bx, \z{i:T}) = q(\bx|\z{i})\prod_{j=i+1}^T p_\theta(\zs{j}|\z{j}).
\end{align}
We note that the optimal decoder is intractable to compute. 
Therefore, ancestral sampling from this improved generative model is infeasible. 
Still, we are going to write out the ELBO and show that we can use it for training the denoiser. 
Similar to \Cref{eq:elbo}, the ELBO for the new generative model takes the form
\begin{align} \label{eq:optimal-decoder-elbo}
    \mathcal{L}^{(i+1)}(\bx) &\triangleq \mathbb{E}_{q(\z{i}|\bx)}[\log q(\bx|\z{i})] - \KL{q(\z{T}|\bx)}{p(\z{T})} \notag \\
    & - \sum_{j=i+1}^T \mathbb{E}_{q(\z{j}|\bx)} [\KL{q(\z{j-1}|\z{j}, \bx)}{p_\theta(\z{j-1}|\z{j})}].
\end{align}
Looking at \Cref{eq:optimal-decoder-elbo}, we observe that the optimal decoder only appears in the first term and is constant with respect to the denoiser parameters $\theta$. 
Therefore, although it is not straightforward to draw samples from the improved generative model, we can still leverage its ELBO to train the denoiser parameters. 

Intuitively, the larger $i$ is, the more ``optimal''  transition steps we use in our reverse process which entails a better model. 
Interestingly, we show in the following theorem that using more ``optimal'' transition steps leads to also an improved %
variational lower bound.

\begin{theorem}[Improved lower bounds] \label{thm:improved-lower-bounds}
For $\bx \sim q(\bx)$, $\mathcal{L}^{(i + 1)}(\bx)$ is on average a better lower bound\footnote{Note that we are comparing the lower bounds for slightly different model distributions (the generative model used in $\mathcal{L}^{(i)}$ has one more reverse transition parameterized by the denoiser than $\cL^{(i+1)}$). This is similar to the argument that the continuous-time ELBO of diffusion models is ``tighter'' than the discrete-time ELBO~\citep{kingma2021variational}, since they capture different reverse models.} than $\mathcal{L}^{(i)}(\bx)$:
\begin{align}
    \E_{q(\bx)}[\mathcal{L}^{(i + 1)}(\bx)] \geq \E_{q(\bx)}[\mathcal{L}^{(i)}(\bx)].
\end{align}
Since $\KL{q(\bx)}{p_\theta(\bx)} = -\mathbb{E}_{q(\bx)}[\log p_\theta(\bx)]  + \text{const} \leq -\mathbb{E}_{q(\bx)}[\cL(\bx)]$, 
this also implies that incorporating an additional optimal reverse transition step results in a smaller upper bound on the KL divergence between the data and model distributions.
\end{theorem}

\begin{proof}
We first take the difference between the two ELBOs as
\begin{align} \label{eq:l-diff}
    \cL^{(i+1)}(\bx) - \cL^{(i)}(\bx) &= \E_{q(\z{i}|\bx)}[\log q(\bx|\z{i})] - \E_{q(\z{i-1}|\bx)}[\log q(\bx|\z{i-1})] \notag \\
    &+ \E_{q(\z{i}|\bx)}[\KL{q(\z{i-1}|\z{i}, \bx)}{p_\theta(\z{i-1}|\z{i})}].
\end{align}
Next, we rewrite the first term by treating the term inside expectation as the log marginal likelihood of the joint model $q(\bx|\z{i-1})q(\z{i-1}|\z{i})$:
\begin{align*}
    &\E_{q(\z{i}|\bx)}[\log q(\bx|\z{i})] 
    = \E_{q(\z{i}|\bx)}\E_{q(\z{i-1}|\z{i}, \bx)}\left[\log \frac{q(\bx|\z{i-1})q(\z{i-1}|\z{i})}{q(\z{i-1}|\z{i}, \bx)}\right] \\
    &= \E_{q(\z{i-1}|\bx)}[\log q(\bx|\z{i-1})] + \E_{q(\z{i}|\bx)}\E_{q(\z{i-1}|\z{i}, \bx)}\left[\log \frac{q(\z{i-1}|\z{i})}{q(\z{i-1}|\z{i}, \bx)}\right].
\end{align*}
The first identity holds due to the fact that $q(\z{i-1}|\z{i}, \bx)$ is the posterior distribution of the joint model $q(\bx|\z{i-1})q(\z{i-1}|\z{i})$ and that the ELBO is tight with the true posterior distribution as the variational distribution.
Plugging this into \eqref{eq:l-diff},  we have
\begin{align*}
    \cL^{(i+1)}(\bx) - \cL^{(i)}(\bx) &=  \E_{q(\z{i}|\bx)}\E_{q(\z{i-1}|\z{i}, \bx)}\left[\log \frac{q(\z{i-1}|\z{i})}{q(\z{i-1}|\z{i}, \bx)} + \log \frac{q(\z{i-1}|\z{i}, \bx)}{p_\theta(\z{i-1}|\z{i})}\right] \\
    &= \E_{q(\z{i-1}, \z{i}|\bx)}\left[\log \frac{q(\z{i-1}|\z{i})}{p_\theta(\z{i-1}|\z{i})}\right], %
\end{align*}
and subsequently
\begin{align*}
    \E_{q(\bx)}[\cL^{(i+1)}(\bx) - \cL^{(i)}(\bx)] &= \E_{q(\bx)}\E_{q(\z{i-1}, \z{i}|\bx)}\left[\log \frac{q(\z{i-1}|\z{i})}{p_\theta(\z{i-1}|\z{i})}\right] \\
    &= \E_{q(\z{i-1}, \z{i})}\left[\log \frac{q(\z{i-1}|\z{i})}{p_\theta(\z{i-1}|\z{i})}\right] \\
    &= \E_{q(\z{i})}[\KL{q(\z{i-1}|\z{i})}{p_\theta(\z{i-1}|\z{i})}] \geq 0.
\end{align*}
Rearranging the terms concludes the proof.
\end{proof}

The above result shows that by using more optimal transitions leads to losses that are tighter bounds on the KL divergence between the data and model distributions.
On the other hand, it is important to note that ancestral sampling is infeasible in the part of the generative process that is replaced by optimal decoders. 
This causes a fundamental tradeoff between the tightness of the loss and the difficulty of simulating a sample -- the more ``optimal'' steps we use, the less analytically tractable transitions we can simulate in the reverse process, as visualized in \Cref{fig:objective}.

One might suggest that we 
can approximately simulate from these ``optimal''  reverse transitions 
by using the trained denoiser 
even if it is not the reverse process used in training. 
However, for a certain $i$ this assumes the denoiser must generalize on the low noise level samples it has never encountered during training.
Therefore, it is natural to ask whether we can construct an objective function that not only leverages the improved ELBOs in \eqref{eq:optimal-decoder-elbo} but also allows us to generate through the denoiser-parameterized reverse model at all timesteps. 
As we shall see next, the widely-used reweighted objective for diffusion models is an example of such improved objective.

\section{Reweighted Losses as Improved Variational Bounds}

We show in the following theorem that typical diffusion model objectives, often reweighted versions of the ELBO  \eqref{eq:elbo}, can be expressed as a weighted sum of the improved variational bounds $\cL^{(i)}$ plus a constant shift. 

\begin{theorem}[Reweighted objectives as improved variational bounds] \label{thm:reweighted-obj}
Let $\tilde{w}(t)$ be a continuous function such that, for all $t(j)$, its value is defined as
$
    \tilde{w}(t(j)) \triangleq \tilde{w}_j = \sum_{i=1}^j w_i
$. Then, we have
\begin{align} \label{eq:diffusion-obj-as-improved-bounds}
   \lim_{T\to \infty} \sum_{i=1}^T w_i \cL^{(i)}(\bx) = \cL^{\tilde{w}}(\bx) + \text{const},
\end{align}
where $\cL^{\tilde{w}}$, $\cL^{(i)}$ are defined as in \Cref{eq:optimal-decoder-elbo,eq:weighted-loss-eps}.
\end{theorem}

\begin{proof}
First, recall that the diffusion ELBO with optimal decoders can be written as
\begin{align*}
    \cL^{(i)} = -\sum_{j=i}^T \cL_{\kl}^{(j)}  + c_i \text{\quad for\quad} \cL_{\kl}^{(j)} \triangleq \E_{q(\z{j}|\bx)}[\KL{q(\zs{j}|\z{j},\bx)}{p_\theta(\zs{j}|\z{j})}],
\end{align*}
where $c_i = \mathbb{E}_{q(\z{i - 1}|\bx)}[\log q(\bx|\z{i - 1})] - \KL{q(\z{T}|\bx)}{p(\z{T})}$ is a constant with respect to denoiser parameters $\theta$. 
Plugging this into the left hand side of \Cref{eq:diffusion-obj-as-improved-bounds}, we have
\begin{align}
    \sum_{i=1}^T w_i \cL^{(i)}(\bx) = -\sum_{i=1}^T w_i \sum_{j=i}^T \cL_{\kl}^{(j)} (\bx) + c
    = -\sum_{j=1}^T \left(\sum_{i=1}^j w_i\right) \cL_{\kl}^{(j)}(\bx)  + c, \label{eq:l-sum}
\end{align}
where $c = \sum_{i=1}^T w_i c_i$ is a constant and we switch the order of summation in the second identity.
Next, we plug in \eqref{eq:kl-t} and take the continuous-time limit
(i.e., $T\to \infty$):
\begin{align*}
    &\lim_{T\to \infty} -\sum_{j=1}^T \left(\sum_{i=1}^j w_i\right) \cL_{\kl}^{(j)} + c \\
    &= \lim_{T\to \infty} -\frac{1}{2}\sum_{j=1}^T  \tilde{w}(t(j)) \frac{\snr(s(j)) - \snr(t(j))}{1/T}\E_{q(\z{j}|\bx)}[\twonorm{\bx - \mu_\theta(\z{j})}^2]\cdot \frac{1}{T} + c \\
    &=  \frac{1}{2}\int_0^1 \tilde{w}(t)  \snr'(t) \E_{q(\bz_t|\bx)}[\twonorm{\bx - \mu_\theta(\bz_t)}^2] \diff t + c.
\end{align*}
Comparing the last line with the definition of $\cL^{\tilde{w}}(\bx)$ concludes the proof.
\end{proof}

\Cref{thm:improved-lower-bounds,thm:reweighted-obj} together shows that $\cL^{\tilde{w}}$ is still a valid variational bound for training the generative model as each component of it lower bounds the data log likelihood up to a constant shift.
Moreover, it improves over the standard diffusion ELBO by leveraging tighter bounds on data-model KL divergence. %
Assigning non-zero weights to $\cL^{(i)}$ with small $i$ values in the weighted sum is also critical, as it ensures the denoiser is exposed to perturbed data at all noise levels, a necessary condition for ancestral sampling with the reverse model to function properly.
The derivation also reveals a requirement for the weighting function: $\tilde{w}(t)$ must be monotonic increasing with respect to $t$ in order for the weights to be positive. 
This aligns with the monotonic condition introduced in \citet{kingma2023understanding} through interpreting the weighting as a cumulative distribution function.

In \Cref{tab:existing-ws}, we list four popular weighting functions used in the  diffusion model literature~\citep{nichol2021improved,karras2022elucidating,lipman2022flow,kingma2023understanding}. 
\citet{kingma2023understanding} expressed these weighting schemes in a reparameterized form $\hat{w}(\lambda)$, where $\lambda \triangleq \lambda(t)$ is the log-SNR.
For convenience, we reproduce their calculations in \Cref{tab:existing-ws} where we explicitly write out the corresponding $\tilde{w}(t)$ form. 
We plot these weighting functions in \Cref{fig:wt-cont-diff} (left).

\begin{figure}[t]
\centering
\includegraphics[height=3.7cm,trim={0 0 7.2cm 0},clip]{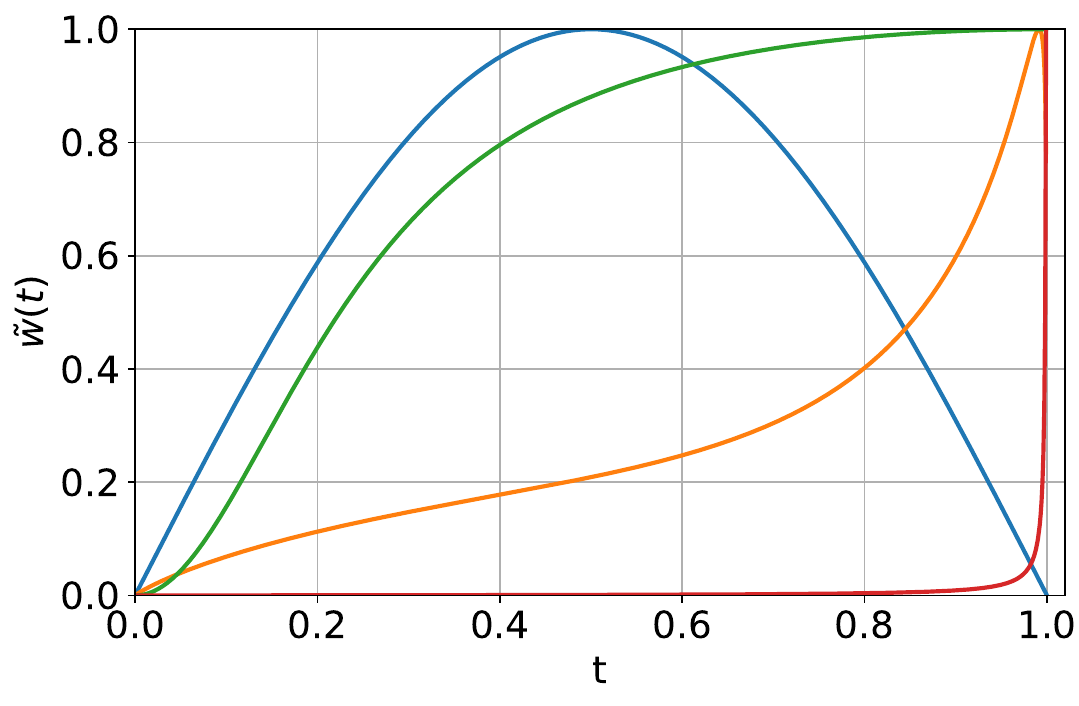}
\includegraphics[height=3.7cm,trim={1cm 0 0 0},clip]{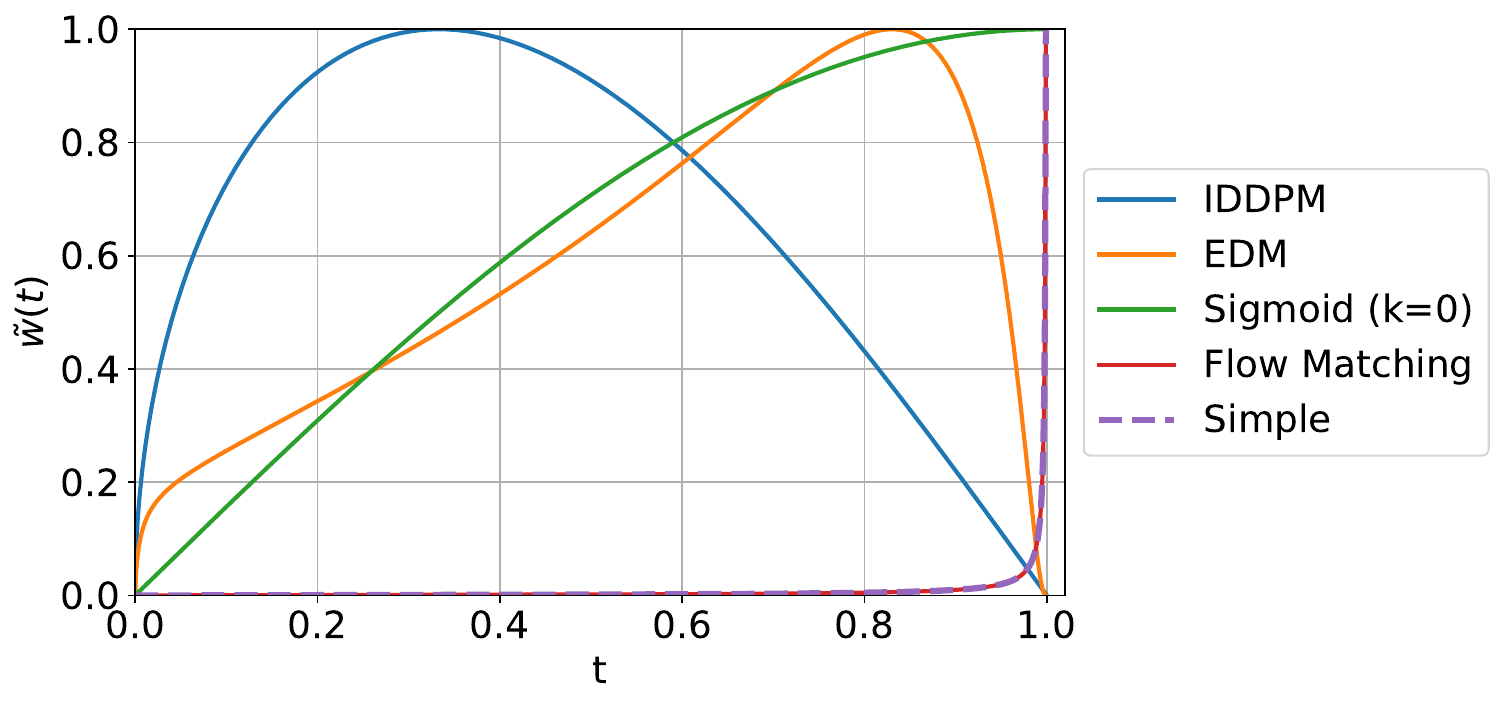}
\caption{Left: Weighting functions used in Gaussian diffusion models. Their formulas can be found in \Cref{tab:existing-ws}. Right: Weighting functions for masked diffusion models, all except the simple weighting are matched from the $w(\lambda)$s of Gaussian diffusion for the cosine schedule $\alpha_t$. All functions are plotted between $[0, 0.999]$ and are normalized with their maximum values in this interval (note that Flow matching and simple weighting approaches infinity at $t=1$). }
\label{fig:wt-cont-diff}
\end{figure}

It is noteworthy that three out of the four weighting functions illustrated in \Cref{fig:wt-cont-diff} exhibit monotonic or near-monotonic behavior (the EDM weighting displays a slight dip approaching time $1$).
The IDDPM weighting is an exception but was also proposed earlier than all other three weightings, which suggests the general practice in this area is converging to monotonic weightings, offering an empirical evidence for the improved variational bound argument we have presented.

\section{Application to Masked Diffusion Models}
\label{sec:discrete-diffusion}

\begin{table}[t] 
\centering
\caption{Weighting functions investigated for masked diffusion models. All functions, excluding the simple weighting, were migrated from continuous diffusion weightings by matching $\hat{w}(\lambda)$. Note that only the sigmoid, flow matching (FM), and simple weightings satisfy the necessary monotonicity requirement when paired with the cosine schedule $\alpha_t = 1 - \cos(\frac{\pi}{2}(1 - t))$.}\label{tab:wlambda-to-discrete}
\vskip 0.1in
\footnotesize
\begin{tabular}{lcccc}
\toprule
Name  & $\lambda(t)$ & $\hat{w}(\lambda)$ & $\tilde{w}(t)$  \\
\midrule
EDM & \multirow{5}{*}{$\log \frac{\alpha_t}{1 - \alpha_t}$} &  $p_{\mathcal{N}(2.4, 2.4^2)}(\lambda)\frac{e^{-\lambda} + 0.5^2}{0.5^2}$ & $w(\lambda(t))$ \\ [4pt]
IDDPM  &  & $\mathrm{sech}(\frac{\lambda}{2})$ & $2\sqrt{\alpha_t(1 - \alpha_t)}$ \\ [4pt]
Sigmoid &  &  $\mathrm{sigmoid}(-\lambda + k)$ & $\frac{1-\alpha_t}{1 - (1 - e^{-k})\alpha_t}$ \\ [4pt]
FM &  &  $e^{-\frac{\lambda}{2}}$ & $\sqrt{\frac{1 - \alpha_t}{\alpha_t}}$  \\ [4pt]
Simple &  &  - & $-\frac{1 - \alpha_t}{\alpha_t'}$  \\
\bottomrule
\end{tabular}
\end{table}

\subsection{Reweighted losses for masked diffusion models}
\label{sec:reweighted-loss-mdm}

Although the preceding sections have been focusing on continuous diffusion models and Gaussian noise -- our theory is more general and agnostic to the choice of diffusion processes. 
We will now illustrate this by deriving the improved variational bounds and reweighted loss specifically for masked diffusion models. 
Unlike continuous diffusion models, a dedicated weighting scheme has not previously been developed for this model class, although empirically people have explored other loss weightings in the context of masked image models~\citep{chang2022maskgit}.

We follow the notations in the MD4 masked diffusion model formulation of \citet{shi2024simplified}.
The forward masked diffusion is a noising process that gradually replaces the data elements with an artificially introduced ``mask'' state $m$.
Important to the characterization of such process is a ``masking schedule'' $\alpha_t$ that determines the expected proportion of unmasked elements at time $t$. 
The joint probability distribution and the discrete-time ELBO follow the same structure as in \Cref{eq:generative-joint,eq:elbo}. 
The KL divergence terms in the ELBO have the following form which is a weighted cross-entropy loss for denoising:
\begin{align}
    \cL_\kl^{(j)} %
    &= -\frac{\alpha_{s(j)} - \alpha_{t(j)}}{1 - \alpha_{t(j)}}\E_{q(\z{j}|\bx)} \left[\delta_{\z{j}, m}\cdot \bx^\top \log \mu_\theta(\z{j})\right].
\end{align}
Repeating the derivation in the proof of \Cref{thm:reweighted-obj} and plugging in the new definition of $\cL_\kl^{(j)}$, we obtain
\begin{align*}
    \cL^{\tilde{w}}(\bx)
    &= \lim_{T\to \infty} \sum_{j=1}^T  \tilde{w}(t(j)) \frac{\alpha_{s(j)} - \alpha_{t(j)}}{1 - \alpha_{t(j)}}\E_{q(\z{j}|\bx)}\left[\delta_{\z{j}, m}\cdot \bx^\top \log \mu_\theta(\z{j})\right] \\
    &=  -\int_0^1  \frac{\tilde{w}(t)\cdot \alpha_t'}{1 - \alpha_t}\E_{q(\bz_t|\bx)}\left[\delta_{\bz_t, m}\cdot \bx^\top \log \mu_\theta(\bz_t)\right] \diff t . 
\end{align*}

The MD4 ELBO corresponds to the special case $\tilde{w}(t) = 1$. 
Therefore, a natural question to ask is which monotonic weighting function can improve the perceptual quality of samples for masked diffusion models.
Below we explore a few potential candidates.

From the above reasoning, we see that the motivation for using a weighted loss in both continuous and masked diffusion models is the same, that is to achieve improved variational bounds. 
This suggests %
to adapt the effective weighting functions $\tilde{w}(t)$ from continuous diffusion models for use in the masked setting. 
However, a potential drawback to directly matching weightings in the time ($t$) space is the lack of reparameterization invariance.
Specifically, \citet{kingma2021variational} noted that the Gaussian diffusion ELBO is invariant to the log-SNR $\lambda(t)$ except its two end points. 
\citet{shi2024simplified} observed the same for masked diffusion models and defined the log-SNR as $\lambda(t) = \log \frac{\alpha_t}{1 - \alpha_t}$.
The reweighted objective written with respect to log-SNR is
\begin{align}
    \mathcal{L}^{\tilde{w}}(\bx) = \int_{-\infty}^{\infty} \tilde{w}(t(\lambda)) \sigma(\lambda) \E_{q(\bz_t|\bx)}\left[\delta_{\bz_t, m}\cdot \bx^\top \log \mu_\theta(\bz_t)\right] \diff \lambda. 
\end{align}
This implies that, if we modify the form of $\lambda(t)$ (or equivalently $\alpha_t$) while keep the two endpoints, the ELBO will stay the same, but the reweighted objective will be significantly different because $\tilde{w}(t(\lambda))$ breaks the invariance.

A potential fix to this problem is to match the $w(\lambda)$ function, with $\lambda$ replaced by masked diffusion's own log-SNR definition.
The new reweighted objective is
\begin{align}
    \cL^{\hat{w}}(\bx) = -\int_0^1  \frac{\hat{w}(\lambda(t))\cdot \alpha_t'}{1 - \alpha_t}\E_{q(\bz_t|\bx)}\left[\delta_{\bz_t, m}\cdot \bx^\top \log \mu_\theta(\bz_t)\right] \diff t.
\end{align}
\Cref{tab:wlambda-to-discrete} summarizes the weighting functions we obtain in this way (in both forms that take $\lambda$ and $t$ as inputs, respectively). 
One interesting case is the sigmoid weighting with $k = 0$, where the loss simplifies to an integration of unweighted cross-entropy losses over $\alpha_t$. 

\begin{figure}[ht]
\centering
\includegraphics[width=0.5\textwidth]{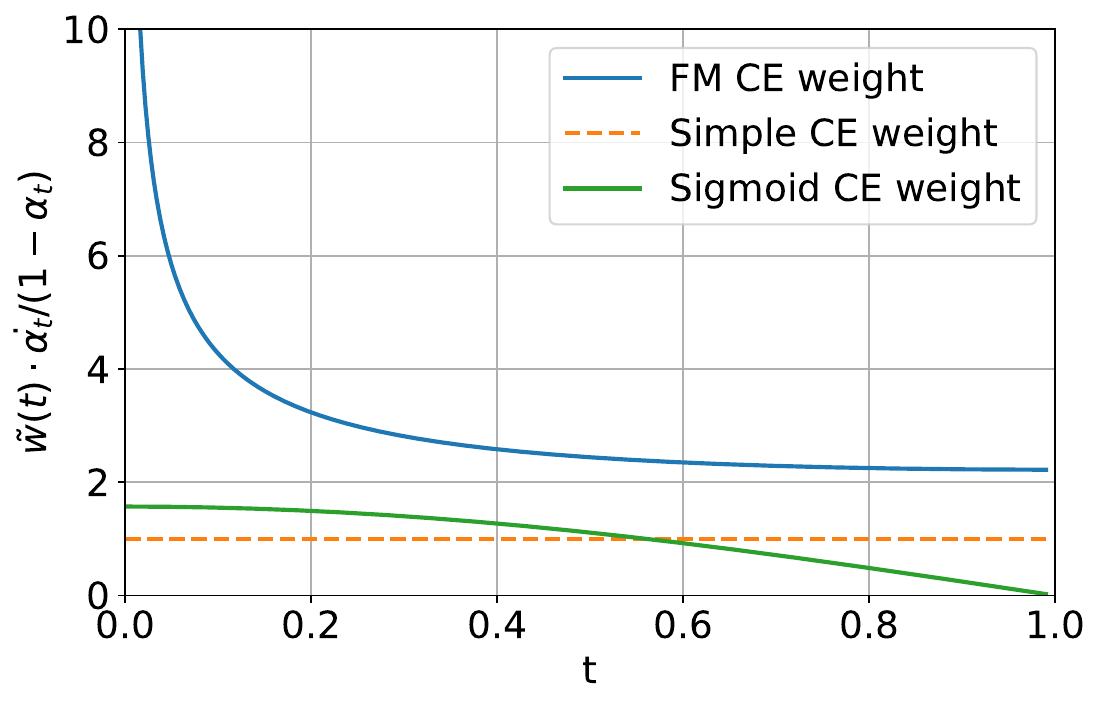}
\caption{Total cross-entropy loss weight under cosine schedule $\alpha_t = 1 - \cos(\frac{\pi}{2}(1 - t))$.}
\label{fig:ce-weight}
\end{figure}

Upon examining the weighting functions in \Cref{fig:wt-cont-diff} (right), we notice that the flow matching weighting exhibits a unique property: it approaches infinity as $t$ approaches 1.
One might consider this singularity a problem, but a further analysis reveals that this behavior balances with the original ELBO weighting for ``slow-start'' schedules like the cosine schedule ($\alpha_t = 1 - \cos(\frac{\pi}{2}(1 - t))$), where unmasking starts slowly in the reverse generation process. 
In this case, $\alpha_t'$ becomes zero at $t=1$, which counteracts the blow-up of $\tilde{w}(t)$ and ensures the total CE weight converges to a finite value at $t=1$, as shown in \Cref{fig:ce-weight}.
Moreover, the CE weight is flat for a large fraction of high noise region between 0.4 and 1.0. 
In contrast, the CE weight for sigmoid schedule vanishes as time approaches $1$.

As we mentioned earlier, without explicitly connecting to diffusion modeling, the literature on masked image models~\citep{chang2022maskgit,li2024autoregressive,li2025fractal} popularized a similar weighted denoising losses where a heuristic weighting scheme is used, i.e., denoising losses on mask inputs are summed over minibatches, divided by the total number of masks in the minibatch. 
If $\mathcal{B}$ denotes the minibatch and $N_{\text{masks}}^i$ the number of masks for data point $i \in \mathcal{B}$,
these methods use the normalization 
$\sum_{i \in \mathcal{B}} N_{\text{masks}}^i$. In contrast, in our masked diffusion objectives we average over the minibatch, i.e., we normalize the weighted sum of the input losses by $|\mathcal{B}|$ (instead of $\sum_{i \in \mathcal{B}} N_{\text{masks}}^i$). However, observe 
that due to the central limit theorem, as the minibatch $\mathcal{B}$ gets large, 
$ \frac{1}{|\mathcal{B}|}\sum_{i \in \mathcal{B}} N_{\text{masks}}^i 
\approx \E[N_{masks}^i]$, where $\E[N_{\text{masks}}^i]$  
is just a constant\footnote{In fact, 
assuming that $t \sim U(0,1)$ this expected value is $\E[N_{\text{masks}}^i] = N p$, %
where $p=\int_0^1 (1 -  \alpha_t) d t$.}.  
This implies that, for large enough minibatch, these  previous approaches
behave as our weighted ELBO objective with a constant CE weight, or equivalently, with  $\tilde{w}(t) = -\frac{1 - \alpha_t}{\alpha_t'}$. 
We call such $\tilde{w}(t)$ \emph{simple} weighting. 
We can check the simple weighting is monotonic for a cosine $\alpha_t$, as shown in \Cref{fig:wt-cont-diff} (right).
Therefore, the simple weighting also induces an improved variational bound and can be applied similarly to masked diffusion training.

\subsection{Evaluation} 
To evaluate whether the reweighted losses for masked diffusion models improve perceptual quality of samples, we conduct a pixel-space class-conditional generation experiment on ImageNet $64\times 64$. 
The experimental setup closely follows MD4~\citep{shi2024simplified} except that we switch the evaluation metric from likelihood (bits-per-dimension) to sample quality metrics including FID~\citep{heusel2017gans} and Inception Distances (IS).

\begin{table}[t]
    \centering %
    \caption{FID score on class-conditional ImageNet 64$\times$64~\citep{karras2022elucidating}. All results of our models are obtained without data augmentation or guidance. Samples are generated with 256 steps using ancestral sampling from the discrete-time reverse process as detailed in \citet{shi2024simplified}.}
    \label{tab:cifar}
    \vskip 0.05in
    \footnotesize
    \begin{tabular}{lrrr}
    \toprule %
     Method & \#Params &  FID ($\downarrow$) & IS ($\uparrow$) \\
    \midrule %
    \textbf{Gaussian Diffusion} \\
    IDDPM \citep{nichol2021improved} & & 2.92 & \\
    ADM \citep{dhariwal2021diffusion} & 296M & 2.07 & \\
    EDM \citep{karras2022elucidating} & 296M & 
    \bf 1.36 & \\
    VDM++ \citep{kingma2023understanding} & 296M &   1.43 & 63.7 \\
    \midrule
    \textbf{Masked Image Models} \\
    MAR \citep{li2025fractal} & 479M & 2.93 \\
    FractalMAR \citep{li2025fractal} & & 2.72 \\
    \midrule
    \textbf{Masked Diffusion} \\
    MD4~(ELBO) & 204M &   6.84 & 30.3 \\
    \textit{Weighting}:  & & & \\
    - IDDPM (non-monotonic)  & 204M &  11.14 & 22.9 \\
    - EDM (nearly-monotonic) & 204M &  4.42 & 37.3 \\ 
    - Sigmoid ($k=0$) & 204M &  3.91 & 40.1 \\
    - FM & 204M & 3.43 & 43.3 \\
    - Simple & 204M & \bf 2.96 & \bf 46.7 \\
    \cmidrule{1-2}
    - Simple & 324M & \bf 1.92 & \bf 57.9 \\
    \bottomrule %
    \end{tabular}
\end{table}

We adopted a network architecture similar to the one used by \citet{shi2024simplified}. The model has 204 million parameters; all architecture and training hyperparameters are summarized in \Cref{tab:train-hyper}.
We test all four choices of weighting functions adapted from continuous diffusion, despite two of them are non-monotonic and therefore not 
compatible with 
our theory.

First, we observed that the extremely non-monotonic IDDPM weighting results in a performance drop compared to standard ELBO. 
This is unsurprising as the non-monotonicity breaks the assumption of positive weights and thus do not lead to valid variational bounds.
In contrast, we observed a notable improvement when switching from the standard ELBO (the original MD4 objective) to the strictly monotonic sigmoid weighting.
We searched the hyperparameter $k$ and observed that $k = 0$ gives best performance among other choices.
The flow-matching (FM) weighting, indicated by our analysis to put significantly more weight on ELBOs that have smaller KL divergences, outperforms the sigmoid weighting.
These results provide strong evidence that our theoretical framework of reweighted objectives is applicable beyond a specific type of diffusion process. 

We also tested the simple weighting function, this time in masked diffusion context. 
Given the similar flatness of the simple weighting and the FM weighting in a wide high-noise time range, and recognizing that the denoising task at small noise regime is relatively straightforward, we expect the simple weighting to also have strong performance.
Aligned with our prediction, the simple weighting achieved competitive FID scores and even outperformed the FM weighting. 
We believe the improvement over FM weighting is due to the further downscaling of the weight at low noise regimes, which helps prevent overfitting on these easy tasks. 
The samples generated from models trained with different weighting functions are visually compared in \Cref{fig:img_samples_204m_monotonic,fig:img_samples_204m_nonmonotic} in Appendix.

Finally, to roughly match the model size used in the continuous diffusion model literature, we further increase the transformer dimension and number of heads, resulting in 325M parameters.
This boosts the FID to 1.92 for the simple weighting, better than continuous diffusion models like IDDPM and ADM.
Although this result remains behind state-of-the-art continuous diffusion models (e.g., EDM), it represents a new record for masked diffusion models on this dataset. 
Class-conditioned samples are shown in \Cref{fig:img_sample_324m} in Appendix.

\section{Conclusion}

Training diffusion models requires accurate approximation of the iterative backward or denoising process and dealing with errors that can  accumulate over time as the process iterates from high to low noise levels. To reduce the effect of  errors we first showed that the  standard ELBO on the data log-likelihood is not the best objective to train the model up to a given denoising time, but instead there is a better time-dependent ELBO having smaller Kullback-Leibler divergence. 
Based on this, we derived a new interpretation of reweighted losses used in Gaussian diffusion and generalized them to masked discrete diffusions. We reported significant improvements in image generation FID scores. For future work, it will be interesting to automate the selection of the weighting for a given data modality, and further extend such methods to 
simultaneously deal with multiple modalities. 

\paragraph{Acknowledgements.} We thank Ruiqi Gao for fruitful discussion on reweighted objectives and Arnaud Doucet for feedback on early drafts.

{\small
    \bibliography{ref}
}

\newpage
\appendix

\begin{table}[t] 
\centering
\caption{Weighting functions used in continuous diffusion models.}\label{tab:existing-ws}
\vskip 0.1in
\footnotesize
\begin{adjustbox}{max width=\textwidth}
\begin{tabular}{lcccc}
\toprule
Name & Parameterization & $\lambda(t)$ & $\hat{w}(\lambda)$ & $\tilde{w}(t)$  \\
\midrule
EDM  & mean prediction & $F_{\mathcal{N}(2.4, 2.4^2)}^{-1}(1 - t)$ &  $p_{\mathcal{N}(2.4, 2.4^2)}(\lambda)\frac{e^{-\lambda} + 0.5^2}{0.5^2}$ & $w(\lambda(t))$ \\ [4pt]
IDDPM & $\epsilon$ prediction & $-2\log\tan(\frac{\pi}{2}t)$ & $\mathrm{sech}(\frac{\lambda}{2})$ & $2\sin(\frac{\pi}{2}t)\cos(\frac{\pi}{2}t)$  \\ [4pt]
Sigmoid & $\epsilon$ prediction & $-2\log\tan(\frac{\pi}{2}t)$ &  $\mathrm{sigmoid}(-\lambda + k)$ & $\frac{1}{1 + e^{-k}\tan(\frac{\pi}{2}t)^{-2}}$\\ [4pt]
FM & velocity prediction & $2\log\frac{1 - t}{t}$ &  $e^{-\frac{\lambda}{2}}$ & $\frac{t}{1 - t}$ \\
\bottomrule
\end{tabular}
\end{adjustbox}
\end{table}

\begin{figure}[p]
\centering
\includegraphics[width=0.65\linewidth]{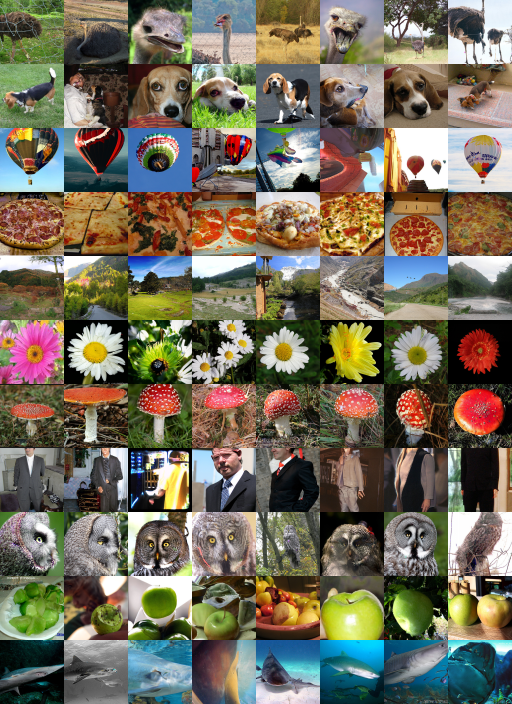}\\
\caption{
Class-conditional samples generated in 256 steps by the masked diffusion model (324M) trained with the simple weighting on ImageNet 64$\times$64 (FID: 1.92). Each row shows samples conditioned on a unique class. We observed a strong diversity in each class, showing good coverage of the data distribution. 
} \label{fig:img_sample_324m}
\end{figure}

\begin{figure}[p]
\centering
\makebox[0.475\linewidth]{ELBO, FID: 6.84}\hfill
\hspace{2.5mm}
\makebox[0.475\linewidth]{Sigmoid, FID: 3.91}\\
\includegraphics[width=0.475\linewidth]{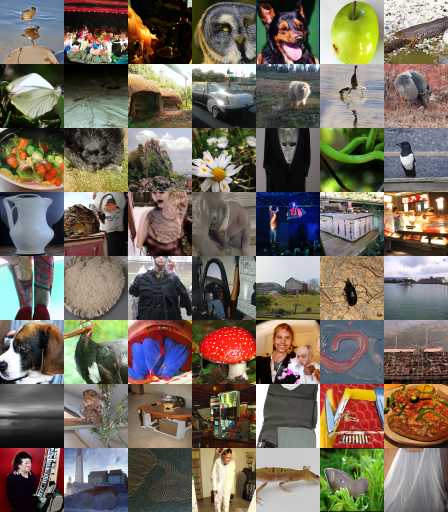}\hfill%
\hspace{2.5mm}\includegraphics[width=0.475\linewidth]{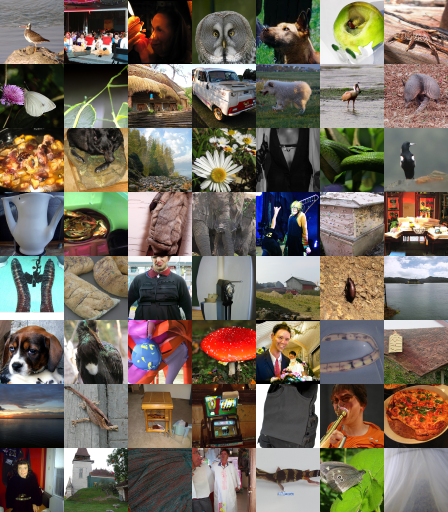}\\
\makebox[0.475\linewidth]{FM, FID: 3.43}\hfill
\hspace{2.5mm}
\makebox[0.475\linewidth]{Simple, FID: 2.96}\\
\includegraphics[width=0.475\linewidth]{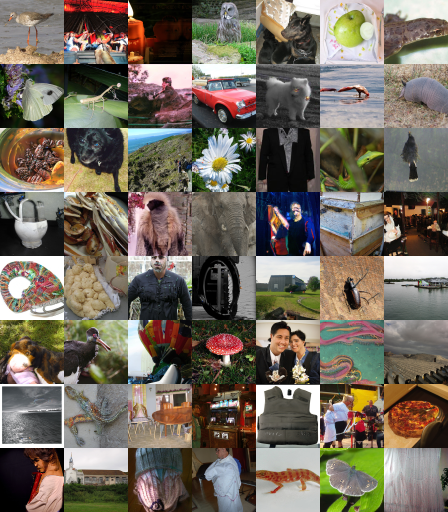}\hfill
\hspace{2.5mm}\includegraphics[width=0.475\linewidth]{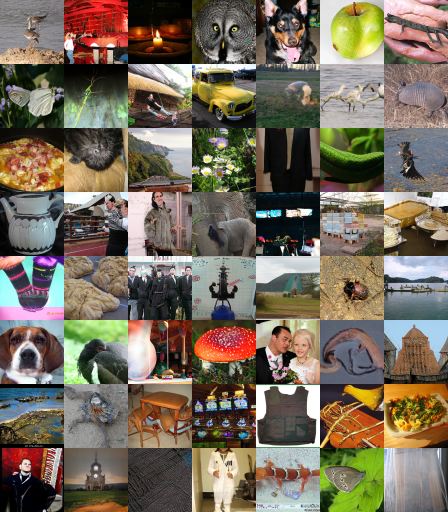}\\
\caption{
Class-conditional generation from masked diffusion models with 204M parameters trained with monotonic weighting functions (ELBO, Sigmoid, FM, Simple) on ImageNet 64$\times$64. Each image is conditioned on a unique class.
} \label{fig:img_samples_204m_monotonic}
\end{figure}

\begin{figure}[p]
\centering
\makebox[0.475\linewidth]{IDDPM, FID: 11.14}\hfill
\hspace{2.5mm}
\makebox[0.475\linewidth]{EDM, FID: 4.42}\\
\includegraphics[width=0.475\linewidth]{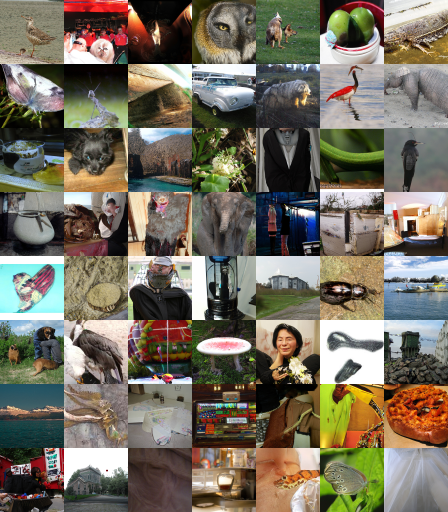}\hfill
\hspace{2.5mm}\includegraphics[width=0.475\linewidth]{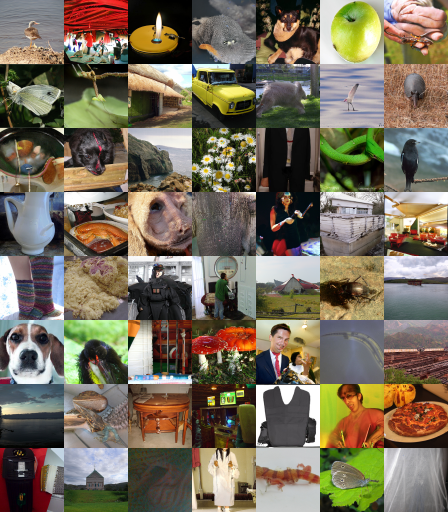}\\
\caption{
Class-conditional generation from masked diffusion models with 204M parameters trained with non-monotonic weighting functions (IDDPM, EDM) on ImageNet 64$\times$64. 
} \label{fig:img_samples_204m_nonmonotic}
\end{figure}

\begin{table}[p]
\centering
\footnotesize
\caption{
Training and network architecture hyperparameters for the 204M and 324M models. The 204M network architecture is the same as \citet{shi2024simplified}'s ImageNet $64 
\times 64$ experiment except that we removed time conditioning~\citep{ou2024}, introduced gating in MLPs and reduced the hidden dimension to 2048. For the 324M network we further replaced original fixed sine-cosine positional encoding in DiT~\citep{peebles2023scalable} with 2D ROPE~\citep{lu2024fit} besides increasing the network size.
} \label{tab:train-hyper}
\vskip 0.1in
\begin{tabular}{lcc}
\toprule
Hyperparameter     & 204M & 324M   \\
\midrule
\textbf{Training} \\
Number of epochs                &  \multicolumn{2}{c}{2M}    \\
Batch size                      & 512             & 1024    \\
Learning rate   & \multicolumn{2}{c}{$2{\times}10^{-4}$} \\
Optimizer  &  \multicolumn{2}{c}{AdamW} \\
Adam $\beta_1$ & \multicolumn{2}{c}{0.9} \\
Adam $\beta_2$ & \multicolumn{2}{c}{0.99} \\
Gradient clipping & \multicolumn{2}{c}{No}      \\
Weight decay & \multicolumn{2}{c}{0.03} \\
LR warm-up steps  & \multicolumn{2}{c}{100}              \\
LR schedule  & \multicolumn{2}{c}{Cosine} \\
EMA          & \multicolumn{2}{c}{0.9999}   \\
Dropout     & \multicolumn{2}{c}{0.1}       \\
\midrule
\textbf{ResNet blocks} \\
Number of blocks (input + output)  & \multicolumn{2}{c}{8 + 8}         \\
Kernel size & \multicolumn{2}{c}{$3\times 3$}        \\
Activation & \multicolumn{2}{c}{SiLU} \\
Number of channels  & \multicolumn{2}{c}{256}     \\
GroupNorm      & \multicolumn{2}{c}{32 groups}       \\
Class conditioning & \multicolumn{2}{c}{AdaLN-zero in GroupNorm} \\
\midrule
\textbf{DiT} \\
Patch size  & \multicolumn{2}{c}{$2\times 2$} \\
Number of blocks  & \multicolumn{2}{c}{20}        \\
Number of heads  & 12 & 16 \\
Head dimension  & \multicolumn{2}{c}{64} \\
Positional encoding & Sine-Cosine & 2D ROPE \\
MLP hidden dimension & 2048 & 2752 \\
Depth scaled init for MLP & \multicolumn{2}{c}{Yes} \\
Gating in MLP & \multicolumn{2}{c}{GLU} \\
Class conditioning & \multicolumn{2}{c}{DiT-style AdaLN-zero} \\
\midrule
\textbf{Others} \\
Input embedding dimension & \multicolumn{2}{c}{256} \\
Class embedding dimension & \multicolumn{2}{c}{256} \\
Time conditioning & \multicolumn{2}{c}{No} \\
Use bfloat16    & \multicolumn{2}{c}{Yes} \\
\midrule
\end{tabular}
\end{table}

\end{document}